\documentclass{article}
\usepackage{arxiv}
\usepackage[american]{babel}
\usepackage[utf8]{inputenc} % allow utf-8 input
\usepackage[T1]{fontenc}    % use 8-bit T1 fonts
\usepackage{hyperref}       % hyperlinks
\usepackage{url}            % simple URL typesetting
\usepackage{booktabs}       % professional-quality tables
\usepackage{amsfonts}       % blackboard math symbols
\usepackage{nicefrac}       % compact symbols for 1/2, etc.
\usepackage{microtype}      % microtypography
\usepackage{lipsum}
\usepackage{bm}
\usepackage{comment}
\usepackage{times}
\usepackage{latexsym}
\usepackage{epsfig}
\usepackage{amsmath}
\usepackage{amssymb}
\usepackage{graphicx}
\usepackage{hyperref}
\usepackage{amsfonts}
\usepackage{physics}
\usepackage{siunitx}
\usepackage{mathtools}
\usepackage{enumitem}
\usepackage{babel}
\usepackage{amsthm}
\usepackage{xcolor}
\usepackage{graphicx,subfigure}
\usepackage[ruled,vlined]{algorithm2e}

\newcommand{\deepexpress}{\textsc{DeepQuant}}
\newcommand{\deepgo}{\textsc{DeepGO}}

\newtheorem{definition}{Definition}
\newtheorem{theorem}{Theorem}
\newcommand{\commentout}[1]{}

\def \one{{\it i)}}
\def \two{{\it ii)}}
\def \three{{\it iii)}}

\newcommand{\localmetric}{Q}

\newcommand{\plainfrac}[2]{#1/#2}

\newcommand{\mesh}{\Delta_k^\text{mesh}}
\newcommand{\pollmesh}{\Delta_k^\text{poll}}
\newcommand{\meshz}{\Delta_0^\text{mesh}}
\newcommand{\pollmeshz}{\Delta_0^\text{poll}}

\newcommand{\D}{\text{\textbf{D}}}
\newcommand{\I}{{\text{\textbf{I}}}}

\newcommand{\nparams}{n}

\newcommand{\poll}{\textsc{poll}}
\newcommand{\search}{\textsc{search}}

\title{Towards the Quantification of Safety Risks in Deep Neural Networks}

\author{
 Peipei XU\thanks{This work is supported by the UK EPSRC projects on Offshore Robotics for Certification of Assets (ORCA) [EP/R026173/1] and End-to-End Conceptual Guarding of Neural Architectures [EP/T026995/1], and ORCA Partnership Resource Fund (PRF) on Towards the Accountable and Explainable Learning-enabled Autonomous Robotic Systems.} \\
  Department of Computer Science\\
  University of Liverpool\\
  Liverpool, L69 3BX, United Kingdom\\
  \texttt{peipei.xu@liverpool.ac.uk} \\
   \And
   Wenjie~Ruan \\
   School of Computing and Communications\\
   Lancaster University\\
   Lancaster, LA1 4WA, United Kingdom\\
   wenjie.ruan@lancaster.ac.uk
   \AND
   Xiaowei Huang \\
   Department of Computer Science\\
   University of Liverpool\\
   Liverpool, L69 3BX, United Kingdom\\
   \texttt{xiaowei.huang@liverpool.ac.uk} \\
}

\begin{document}
\maketitle

\begin{abstract}
Safety concerns on the deep neural networks (DNNs) have been raised when they are applied to critical sectors. 
In this paper, we define safety risks by requesting the alignment of network's decision with human perception. To enable a general methodology for quantifying safety risks, we define a generic safety property and instantiate it to express various safety risks. 
For the quantification of risks, we take the maximum radius of safe norm balls, in which no safety risk exists. 
The computation of the maximum safe radius is reduced to the computation of their respective Lipschitz metrics -- the quantities to be computed. 
In addition to the known adversarial example, reachability example, and invariant example, in this paper we identify a new class of risk -- uncertainty example -- on which humans can tell easily but the network is unsure.
We develop an algorithm, inspired by derivative-free optimization techniques and accelerated by tensor-based parallelization on GPUs, to support an efficient computation of the metrics.
We perform evaluations on several benchmark neural networks, including ACSC-Xu, MNIST, CIFAR-10, and ImageNet networks. 
The experiments show that, our method can
achieve competitive performance on safety quantification in terms of the tightness and the efficiency of computation. Importantly, as a generic approach, our method can work with a broad class of safety risks and without restrictions on the structure of neural networks. 
We release our tool in GitHub: \url{https://github.com/TrustAI/DeepQuant} for the community to use.

\end{abstract}

\section{Introduction} \label{sec:intro}

In recent years, we witness significant progress has been made in AI, especially the deep neural networks that can achieve surprisingly high performance on various tasks, including image recognition \cite{russakovsky2015imagenet}, natural language processing \cite{lecun2015deep}, and games \cite{alphaGoZero}. As a key component, deep neural networks have also been widely used in a range of safety-critical applications such as fully- or semi-autonomous vehicles~\cite{maqueda2018event}, drug discovery~\cite{webb2018deep} and automated medical diagnosis~\cite{esteva2019guide}. 
The applications of neural networks in safety-critical systems bring a new challenge. As recent research demonstrated~\cite{SZSBEGF2014,FGSM}, despite of achieving
high accuracy, DNNs are vulnerable to adversarial examples, i.e., adding a small 
perturbation to a genuine image will result in an erroneous 
output. Such phenomena essentially implies that, neural network's 
accuracy and its robustness may not be positively correlated~\cite{tsipras2018robustness}. As a result, it is extremely crucial that a neural network model can be practically evaluated on its safety and robustness~\cite{weng2018evaluating,SWRHKK2018,huang2019coverage}. 

Many research efforts have been directed towards developing approaches to evaluate neural network's robustness by %in terms of 
crafting adversarial examples~ \cite{athalye2017synthesizing,ilyas2018black,galloway2017attacking,mopuri2018nag,yanghao2020generalizing}, 
including notably FGSM \cite{SZSBEGF2014}, JSMA \cite{JSMA}, C\&W \cite{CW-Attacks}, etc. These approaches can only falsify robustness claims, yet cannot verify,  
because no theoretical guarantee is provided on their results. 
Originated from verification community recently, some research works have instead focused on robustness evaluation with rigorous 
guarantees~\cite{peck2017lower,hein2017formal}, i.e., if no adversarial examples found, the proposed solution can guarantee that DNN's output is invariant to adversarial perturbation. 
These techniques rely on either a reduction to a constraint solving problem by encoding the network as a set of constraints \cite{katz2017reluplex},
an exhaustive search of the neighbourhood of an image \cite{HKWW2017}, or an over-approximation method~\cite{gehr2018ai}, etc. However, these approaches can only work with small-scale neural networks in a white-box manner\footnote{Namely, the structure and the internal weights of DNNs need to be known}, and have not been able to work with a practical
state-of-the-art neural networks such as 
various ImageNet models. Moreover, most of them are dedicated for a particular single safety risk such as local or point-wise robustness. Please refer to our recent survey for details~\cite{safetysurvey}.

In this regard, this paper works towards a generic quantification framework that is able to
\one~work with different classes of safety risks; \two~provide guarantee 
on its quantification results; and \three~applicable to large-scale neural networks with a broad range of layers and activation functions. 
To achieve these goals, we introduce a generic property expression parameterised over the output of a DNN, define metrics over this expression, and 
develop a tool \deepexpress\ to evaluate the metrics on DNNs. By instantiating the property expression with various specific forms and consider different metrics, \deepexpress\ can evaluate different safety risks on neural networks including 
the local and global robustness, as well as the decision uncertainty, a new type of safety risks that is firstly studied in this paper.
Specifically, the key technical contributions of this paper lie on the following aspects.

First, we study safety risks by assuming that network's decision needs to align with human perception. Under this assumption, we identify another class of safety risks other than the known ones -- adversarial example \cite{SZSBEGF2014}, reachability example \cite{dutta2017output,RHK2018}, and invariant example \cite{jacobsen2020exploiting} -- and name it as \textbf{uncertainty example}. 
Fig.~\ref{fig-1} presents the intuition of these safety risks.
Different from adversarial example on which the network is certain about its decision (although the decision is incorrect w.r.t. human perception), uncertainty example lies on the vicinity of the intersection point of all decision boundaries (marked by red dashed line circle in Fig.~\ref{fig-1}) and should be without any confusion with human perception. uncertainty are more difficult to evaluate than robustness because the intersection areas of \emph{all} decision boundaries are very sparse in the input space.
The potential disastrous consequence of uncertainty example will be  discussed in the paper.

Second, to work with different safety risks in a 
framework, we introduce a \textbf{generic safety property expression} and show that it can be instantiated to express various risks. 
The quantification of the risks is then defined as the maximum  radius of safe norm balls, in which no risk is present. Then, we show that, a
\emph{conservative} estimation of the maximum radius can be done by computing a \textbf{Lipschitz metric} over the safety property.  

Third, we develop an algorithm, inspired by a derivative-free optimisation technique called Mesh Adaptive Direct Search, to compute the Lipschitz metric. The algorithm is able to work on large-scale neural networks and does not require to know the internal weights or structures of DNNs. 
Moreover, as indicated in Fig. \ref{fig-e3}, our algorithm is tensor-based, to take advantage of the significant capability of \textbf{GPU parallelisation}. 

Finally, we implement the approach into a tool  \deepexpress\footnote{The software is provided via github:~\url{https://github.com/TrustAI/DeepQuant}} and validate it over an extensive set of networks, including large-scale ImageNet DNNs with millions of neurons and tens of layers. The experiments show competitive performance of \deepexpress\ in a number of benchmark networks  with respect to state-of-the-art tools ReluPlex \cite{katz2017reluplex}, SHERLOCK \cite{dutta2017output}, and DeepGO \cite{RHK2018}: it is able to \textbf{efficiently} achieve \textbf{tight} bounds. Other than the performance, our method can work without restrictions on the safety properties and the structure of neural networks. This is in contrast with existing tools, for example ReluPlex and SHERLOCK can only work with small network with ReLU activation functions and DeepGO can only work with robustness and reachability. 
In summary, the novelty of this paper lies on the following aspects:

\begin{itemize}

\item This paper introduce a generic property expression that provides a principal and unified tool to quantify various safety risks on deep neural networks.

\item We theoretically prove that the proposed Lipschitzian robustness expression bounds the true robustness in terms of classification-invariant space.
\item This paper, as the the first research work, identifies a new type of risk of neural networks by uncertainty examples, as well as provides an efficient method to locate such uncertainty spots.

\item We implement the proposed solution as a software tool - \deepexpress\ that is applicable to large-scale deep neural networks including various ImageNet models.
\end{itemize}

\begin{figure*}[ht]
\centering
	    \centering
	    \begin{minipage}{0.4\linewidth}
	        \centering
		    \includegraphics[width=1\linewidth]{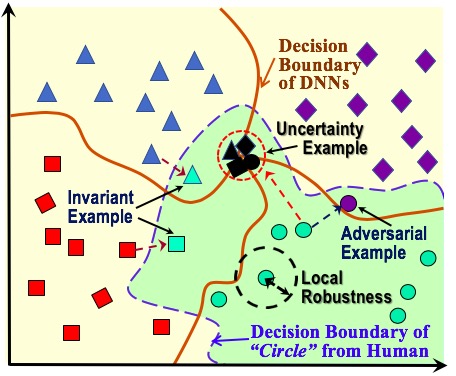}
		    \text{(a)}
	    \end{minipage}%
	    \hspace{3mm}
	    \begin{minipage}{0.53\linewidth}
		    \centering
		    \includegraphics[width=1\linewidth]{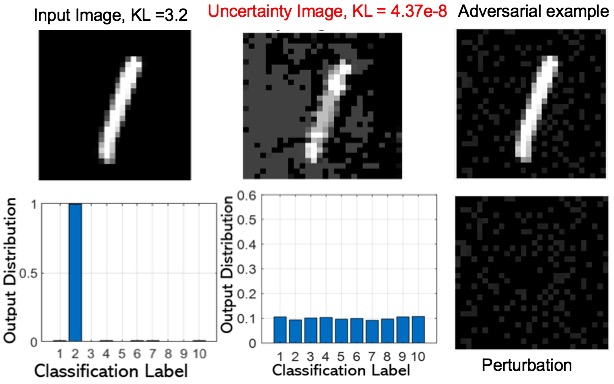}
		    \text{(b)}
	    \end{minipage}
	\centering
	\caption{{\bf(a)} Illustration of three safety risks - adversarial example \cite{SZSBEGF2014}, invariant example \cite{jacobsen2020exploiting}, and uncertainty example (\textbf{this paper}). 
	{\bf~~(b)} An example to compare uncertainty example with adversarial example in MNIST. The First Row: the first image is the raw input image, the second image is the uncertainty example (identified by our tool) and the third image is the adversarial example; The Second Row: the corresponding output probabilistic distributions of DNNs on raw input image and uncertainty example, and the adversarial perturbation.
	}
    \label{fig-1}
\end{figure*}

\section{Related Work} 
\label{sec:related}
We now discuss some of the closely related work in safety properties of neural networks.

\subsection{Adversarial Attacks}

As recent works show that DNNs are vulnerable to adversarial examples, there are constantly increasing number of attacks to generate adversarial examples with new countermeasures~\cite{tramer2018ensemble}. 
Adversarial attacks apply heuristic search algorithms to find adversarial examples. 
Starting from Limited-memory Broyden-Fletcher-Goldfarb-Shanno (L-BFGS) algorithm~\cite{SZSBEGF2014}, 
a number of adversarial attack algorithms have been developed, including notably FGSM \cite{FGSM}, JSMA \cite{JSMA}, C\&W attacks \cite{CW-Attacks}, RecurJac \cite{zhang2018recurjac}, one-pixel attacks~\cite{ su2019one}, structured attack~\cite{ xu2018structured}, binary attack~\cite{galloway2017attacking} etc.

Most of current works are guided by the forward gradient or the gradient of the cost-function, which in turn rely on the existence of first-order derivative, i.e., differentiability, of neural network. The method proposed in this paper relaxes this assumption and can work with any neural network.   
Moreover, while adversarial attacks can falsify the robustness of a neural networks, 
our method can also verify the robustness, thanks to its theoretically grounded approach of taking a Lipschitzian metric with confidence interval expression as an indicator of the robustness. Finally, beyond robustness, our metric is generic and can express other properties such as Uncertainty.

\subsection{Safety/Formal Verification}

How to verify whether a given/particular neural network satisfies certain input-output properties is a very challenging task. Traditional verification of neural networks mainly focus on measuring the networks on a large collections of points in the input space and checking whether the outputs are as desired. However, due to the infinite of input space, it is not workable to check all possible inputs. Some networks may be vulnerable to adversarial attacks, although they can perform well on a large sample of inputs and not correctly extend to new situations.
The recent advances of neural network verification include 
the layer-by-layer exhaustive search approach~\cite{HKWW2017}, methods using constraint solvers~\cite{PT2010,katz2017reluplex}, global optimisation approaches~\cite{WHK2017,RHK2018,WWRHK2018,RWSHKK2019}, the abstract interpretation approach \cite{gehr2018ai,mirman2018differentiable,LLYCHZ2019}, 
linear programming (LP)~\cite{Wong2018ConDual} or mixed-integer linear programming (MILP)~\cite{MILP2017Tjing},
semi-definite relaxations~\cite{raghunathan2018semidefinite},
Lipschitz optimization~\cite{weng2018towards, BSZ2017}, and combining optimization with abstraction~\cite{Anderson2019OptimizationAA}. 
The properties studies include robustness \cite{HKWW2017,weng2018towards}, reachability (i.e., whether a given output is possible from a given subspace of inputs)~\cite{dutta2017output}, and properties expressible with SMT constraints \cite{PT2010,katz2017reluplex}. 

Verification approaches aim to provide guarantees on the obtained results. However, they cannot provide efficient solutions to large-scale neural networks. For example, constraint-based approaches such as Reluplex 
can only work with neural networks with a few hundreds hidden nodes~\cite{PT2010,katz2017reluplex,LM2017}.
%,LM2017}. 
Exhaustive search and global optimisation suffer from the state-space or dimensionality explosion problem~\cite{HKWW2017,RHK2018}.
Different from these solutions, the  quantification method proposed in this paper can work efficiently on large-scale neural networks against Lipschitzian properties.

\section{Safety Risks in Neural Networks} \label{sec:LipschitzMetric}

A (feed-forward and deep) neural network can be represented as a function $f: \mathbb{R}^n \rightarrow \mathbb{R}^m$ such that given an input $x\in \mathbb{R}^n$, it outputs a probabilistic distribution over a set of $m$ labels $\{1...m\}$, representing the probabilities of assigning labels to the input.  
We use $f_j(x)$ to denote the probability of labelling an input $x$ with the label $j$. Based on this, we define the labelling function $l:\mathbb{R}^n \rightarrow \{0...m\}$ as 
\begin{equation}
    l(x) = \left\{
    \begin{array}{ll}
      k ~~~~& |f_k(x) - \max_{j\neq k}f_j(x)| > \epsilon \\
      0 & \text{otherwise}
    \end{array}
    \right.
\end{equation}
where $k=\max_jf_j(x)$ is the label with the greatest confidence and $\epsilon$ is a threshold value. Intuitively, if there is a label $k \in \{1...m\}$ with significant confidence comparing to other labels $j\neq k$, we assign $x$ with the label $k$. On the other hand, if there is no label with significant confidence comparing to other labels, we assign $x$ with the label $0$, denoting that the network is not confident about its own decision.

In practice, a neural network 
is a complex, highly nonlinear function composed of a sequence of simple, linear or nonlinear functional mappings~\cite{Bishop2006, Goodfellow2016}. Typical functional mappings  
include fully-connected, convolutional, pooling, Softmax, and Sigmoid. In this paper, we treat the network as a blackbox and therefore can work with any internal layer and architecture as long as the network is feedforward.

\textit{Safety Risk:}
By training over a labelled dataset, a network $f$ is to simulate the decisions of a human $O: \mathbb{R}^n\rightarrow \{0...m\}$ on unseen inputs, where $O(x)=0$ represents that the human cannot decide on its labelling. Therefore, the safety risk of $f$ lies on the inconsistency of decisions between $f$ and $O$, as defined in Definition~\ref{def:safetyrisk} and  Definition~\ref{def:reachabilityexample}. 

\begin{definition}[Misalignment on Decision]\label{def:safetyrisk}
	Given a network $f: \mathbb{R}^n \to \mathbb{R}^m$, a human decision oracle $O: \mathbb{R}^n \to \{0..m\}$, and a legitimate input $x\in \mathbb{R}^n$ such that $l(x)=O(x) \neq 0$, we have the Table-\ref{tab:my_label} for $\hat x$ being another input that is perturbed from $x$. 

\begin{table*}[ht]
    \caption{Categories of safety risks by 
    the alignment of neural network decisions with human perception. 
    %Among all risks,  
    \textbf{Uncertainty examples} are for the first time studied in this paper.}
    \centering
    \begin{tabular}{|l|c|c|c|}
    \hline
         & $O(\hat x) = 0$ & $O(\hat x) = O(x)$ & $0 \neq O(\hat x) \neq O(x)$  \\
         \hline
        $l(\hat x) = 0$ & no error & \textbf{Uncertainty example}  & \textbf{Uncertainty example} \\ 
        $0\neq l(\hat x) = l(x)$ & adversarial example \cite{SZSBEGF2014} & no error & invariant example \cite{jacobsen2020exploiting}\\
        $0 \neq l(\hat x) \neq l(x)$ & adversarial example \cite{SZSBEGF2014} & adversarial example \cite{SZSBEGF2014} & no error \\
        \hline
    \end{tabular}
    \label{tab:my_label}
     %\vspace{-4mm}
\end{table*}

\end{definition}

Intuitively, each entry in Table-\ref{tab:my_label} represents a possible scenario for $x$ and $\hat x$. For example, those entries on the diagonal represent that no obvious error can be inferred. For the case where $0 \neq l(\hat x) \neq l(x)$ and $O(\hat x) = O(x)$, human believes that the two inputs are in the same class but the network believes not, representing a typical case of adversarial example \cite{SZSBEGF2014}. 
The two entries with $O(\hat x)=0$ represent the scenarios where human is uncertain about $\hat x$ while the network has high confidence about it. They are also seen as adversarial examples. Moreover, invariant example \cite{jacobsen2020exploiting} occurs when $ x$ and $\hat x$ are labelled as the same while human believes they should belong to different classes.
Finally, \textbf{uncertainty example}, to be discussed for the first time in this paper, covers two entries where the network is uncertain when human can clearly differentiate.

Uncertainty may lead to safety concern in practice. For example, it has been well discussed that adversarial examples \cite{SZSBEGF2014} may lead to disastrous consequences. For example, in a shared autonomy scenario where a human driver relies on a deep learning system to make most of the decisions and expects its handing over of the control only when necessary, the deep learning system may act confidently (i.e., $l(\hat x)\neq 0$) when human believes that it should perform the other action (i.e., $O(\hat x)=O(x)=l(x)\neq l(\hat x)$) or ask for the transfer of control back to human (i.e., $O(\hat x)=0$). These are adversarial examples. 
On the other hand, the Uncertainty example suggests the other serious consequence: 
it is possible that the deep learning system intends to hand back the control (since $l(\hat x)=0$) while the human driver believes the deep learning is able to handle it very well and loses her concentration (cf. Tesla incident and Uber incident). 

Besides the risks from the mis-alignment of prediction decisions (i.e., adversarial example, invariant example, and uncertainty example), we have the following:  

\begin{definition}[Misalignment on Rigidity of Classification Probability]\label{def:reachabilityexample}
Given $f_j(x)$ and a pre-specified constant $\epsilon$, it is possible that human may expect the unreachability of $f_j(x)+\epsilon$ under certain perturbation on $x$, while neural network can. We call those perturbed inputs $\hat x$ that satisfy $f_j(\hat x) \geq f_j(x)+\epsilon$  \textbf{reachability examples}. 
\end{definition}

\textit{Norm Ball:}
In Definition~\ref{def:safetyrisk}, we use ``$\hat x$ being another input that is perturbed from $x$'' to state that $\hat x$ is close to $x$. This is usually formalised with norm ball as follows. 
\begin{equation}
    \mathcal{B}(x,d,p) = \{\hat x |~||\hat x-x||_p \leq d\}
\end{equation}
Intuitively, $\mathcal{B}(x,d,p)$ includes all inputs that are within a certain distance to $x$. The distance is measured with $L_p$-norm such that $||x||_p=(\sum_{i=1}^{n} x_i^p)^{1/p}$. The ``certain perturbation on $x$'' in Definition~\ref{def:reachabilityexample} is also formalised in this way.

\section{Quantification of Safety Risks}\label{SafetyExpression}

In this paper, we consider three safety risks: adversarial example, uncertainty example, and reachability example. 
First of all, we take a generic definition of safety property.

\begin{definition}\label{def:property}
	A safety property  $s(x)$ is an 
	expression over the outputs $\{f_i(x)~|~i\in \{1...m\}\}$ 
	of the neural network, and we expect that whenever  $s(x) < 0 $, the neural network has safety risk.
\end{definition}
In the following, we show how to instantiate $s(x)$ with 
\textit{specific expressions} in order to quantify the robustness, the reachability, and the uncertainty.

\subsection{Robustness Quantification}
\label{sec:robustness}
Firstly, 
%we can robustness norm ball. 
a norm ball  $\mathcal{B}(x,d,p)$ is a \textbf{safe norm ball} if $l(\hat x)=l(x)$ for all $\hat x\in \mathcal{B}(x,d,p)$. 
Moreover, a norm ball  $\mathcal{B}(x,d,p)$ is a \textbf{targeted safe norm ball} w.r.t. a pre-specified label $l$ if $l(\hat x)\neq l$ for all $\hat x\in \mathcal{B}(x,d,p)$. Intuitively, a safe norm ball requires all the inputs within it to have the same label as the center point $x$, while a targeted safe norm ball is to avoid having any input to have a specific label $l$. 

Based on safe norm balls, we define the robustness as below.

\begin{definition}[Robustness]\label{def:localrobustness}
	Given a network $f$, an input $x$, and a norm ball $\mathcal{B}(x,d,p)$, the  robustness of $f$ on $x$ and $\mathcal{B}(x,d,p)$ is to find the maximum radius $d'$ that can make $\mathcal{B}(x,d',p)$ safe. More specifically, $\mathcal{B}(x,d',p)$ is a safe norm ball, and for all $d''>d'$, $\mathcal{B}(x,d'',p)$ is not a safe norm ball. We use $R(x,d,p) $ to denote such a maximum safe radius $d'$, and call it robustness radius. 
\end{definition}

It is noted that $R(x,d,p) \leq d$. Intuitively, the robustness of $f$ on $x$ and $\mathcal{B}(x,d,p)$ is evaluated with the maximum radius of safe norm balls, which are  centered at $x$ and within the norm ball $\mathcal{B}(x,d,p)$. 
We remark that, accurately calculating the robustness is extremely difficult in a high-dimensional space, see e.g.,~\cite{SZSBEGF2014,katz2017reluplex}. 

Below, we instantiate the safety property $s(x)$ with {\em Confidence Interval} expression, which can be used to quantify the robustness.

\begin{definition}[Confidence Interval Expression]\label{def:gapfunction}
	Let $f$ be a network, $x$ an input,
	and $l_1, l_2 \in \{1...m\}$ two labels, we define confidence interval expression as follows: 
	\begin{equation}
		s_{CI}(x)(l_1,l_2) = f_{l_1}(x) - f_{l_2}(x) - \epsilon
	\end{equation} 
	where $\epsilon \in [0,1]$ 
	%is a constant specifying 
	specifies the minimum confidence interval required by the user.
	\end{definition}

According to Definition~\ref{def:property}, we use $s(x)<0$ to express the existence of potential risks. Therefore, intuitively, the expression $s_{CI}(x)(l_1,l_2)$ suggests a safety specification that the confidence gap between labels $l_1$ and $l_2$ on input $x$ has to be larger than a pre-specified value $\epsilon$. Depending on the concrete safety requirements, a user may instantiate $l_1$, $l_2$, and $\epsilon$ into different values. 
We can instantiate $l_1$ and $l_2$ and obtain the following concrete confidence-interval expressions: 
	\begin{itemize}
		\item Case-1: $s_{CI}(x)(j_1,j_2)$, where for some other input $x_0\neq x$, $j_1 = \arg\max_{j} f_j(x_0)$ is the label with the greatest confidence value and $j_2 = \arg\max_{j\neq j_1} f_j(x_0)$ is the label with the second greatest confidence value;
		\item Case-2: $s_{CI}(x)(j_1,l)$ for some given label $l$;
		\item Case-3: $s_{CI}(x_0)(j_1,j_m)$, where $j_m = \arg\min_{j} f_j(x_0)$ is the label with the smallest confidence value.
	\end{itemize}

Intuitively, the above expression maintain different types of discrepancies between two confidence values of an input $x$. In particular, the expression $s_{CI}(x)(j_1,j_2)$ in Case-1 is closely related to the resistance of DNNs to untarget adversarial attacks. Expression in Case-2 is reflect the robustness to target adversarial attacks. In both cases, we may use $\epsilon=0$, to denote a mis-classification, or assign $\epsilon$ with some value to make sure that the network mis-classifies with high confidence (a more serious scenario). And expression in Case-3 instead captures the largest variation between confidence values.

While $s_{CI}(x)(j_1,j_2)$ provides an expressible way to specify whether an input \emph{directly} leads to the safety risk, we need to show how to use this expression for the purpose of  evaluating robustness. 
Below, we define a Lipschitzian metric.

\begin{definition}[Lipschitzian Metric]\label{def:localmetric}
	Given an expression $s(x)$, a norm ball $\mathcal{B}(x,d,p)$ centered at an input $x$, we let $\localmetric(s, x, d, p)$ be a Lipschitzian metric, defined as follows.
	\begin{equation}\label{equ:localrobustness}
		\localmetric(s, x, d, p) = \sup_{\hat x \in \mathcal{B}(x,d,p)}\dfrac{|s(x)-s(\hat x)|}{||x-\hat x||_p}
	\end{equation}
\end{definition}
Intuitively, the metric is, based on a given point $x$, to find the greatest changing rate within the norm ball $\mathcal{B}(x,d,p)$. 
The following theorem shows that, the  robustness radius $R(x,d,p) $ can be estimated \emph{conservatively} if the Lipschitzian metric can be computed. 

\begin{theorem}\label{thm1}
Given a neural network $f$, an input $x$, and a norm ball $B(x,d,p)$, 
we have that, $\mathcal{B}(x,d',p)$ is a safe norm ball when $\displaystyle d'=\frac{s(x)}{\localmetric(s, x, d, p)}\leq d$. 
\end{theorem}

\begin{proof}
By the robustness definition in Definition~\ref{def:localrobustness}, we need to have 
\begin{equation}\label{eqn-162}
	\forall \theta: ||\theta||_p \leq d' \Rightarrow s(x+\theta) \geq 0 
\end{equation}
Since neural networks are Lipschitz~\cite{RHK2018}, % not cite?
we have that, for all $ x+\theta \in \mathcal{B}(x,d,p)$, \begin{equation}\label{eqn-16}
	|s(x) - s(x+\theta)| \leq \localmetric(s, x, d, p)~||\theta||_p
\end{equation}
We consider two possible cases: $s(x+\theta) \geq s(x)$ or $s(x+\theta) < s(x)$. For the case of  $s(x+\theta) \geq s(x)$, it is straightforward that $s(x+\theta)\geq 0$, since $s(x)\geq 0$ by the safety requirement. For the  case of $s(x+\theta) <  s(x)$, we have that 
\begin{equation}\label{eqn-163}
		s(x) - \localmetric(s, x, d, p)~||\theta||_p  \leq  s(x+\theta)
\end{equation}
To ensure $s(x+\theta)\geq 0$, it is sufficient to have $s(x) - \localmetric(s, x, d, p)~||\theta||_p  \geq 0 
$. By $||\theta||_p \leq d'$, it is sufficient to have $s(x) - \localmetric(s, x, d, p)~d'  = 0 
$. Therefore, if we have $d'=\plainfrac{s(x)}{\localmetric(s, x, d, p)}$ then Eqn.~(\ref{eqn-162}) holds, i.e., $\mathcal{B}(x,d',p)$ is a safe norm ball.  

Moreover, we require that $d'\leq d$, since otherwise Eqn.~(\ref{eqn-16}) may not hold. Intuitively, this is because the computation of $\localmetric(s, x, d, p)$ is conducted within 
$\mathcal{B}(x,d,p)$, and hence any result based on it may not work over a greater norm ball. $\hspace*{0em plus 1fill}\square$
\end{proof}
The above theorem suggests that, we can use $\plainfrac{s(x)}{\localmetric(s, x, d, p)}$ 
to conservatively estimate the robustness radius $R(x,d,p) $. 
It is known that $s(x)$ is trivial, so the estimation of robustness radius $R(x,d,p) $ is reduced to the estimation of Lipschitz metric $\localmetric(s, x, d, p)$.

\subsection{Uncertainty Quantification}\label{sec:uncertainty}

As explained in Definition~\ref{def:safetyrisk}, adversarial examples -- the risk for robustness -- are not the only class of safety risks. In this section, we study another type of safety risk, i.e., uncertainty examples. To the best of our knowledge, this is the first time this safety risk is studied. 
We remark that, the study of this risk becomes easy, owing to our approach of taking a generic expression $s(x)$. Also, its estimation and detection can take the same algorithm as the robustness quantification. That is, it comes for free. 

Since uncertainty examples represent those inputs on which the network $f$ cannot have a clear decision, we 
need to express the \emph{uncertainty} of the distribution $f(x)$. This can be done by considering the Kullback-Leibler divergence~\cite{KLD-NIPS08} (or KL divergence) from $f(x)$ to e.g., the uniform distribution or another distribution $f(\hat x)$. 

\begin{definition}[Uncertainty Expression]\label{def:klfunction}
	Let $f$ be a network and $x$ an input, we write 
	\begin{equation}\label{equ:uniform}
		s_{U}(x) = - \epsilon -\sum_{l=1}^m\frac{1}{m}\log m f_l(x)
	\end{equation}
	where $\epsilon>0$ is a bound representing, from the DNN developer's view, what is the smallest KL divergence from the uniform distribution for 
	$x$ to be classified as a good behaviour. 
	Moreover, if consider the other distribution $f(\hat x)$ as the basis, 
	we have 
		$s_{U}(x,\hat x) = - \epsilon -\sum_{l=1}^m\frac{1}{m}\log \frac{f_l(x)}{f_l(\hat x)}$.
\end{definition}
Intuitively, the uniform distribution indicates that the network is unsure about the input.
Therefore, in Eqn. (\ref{equ:uniform}), we require as a necessary condition, for the decision on $x$ to be safe, that the KL divergence from $f(x)$ to the uniform distribution (expressed as $-\sum_{l}\frac{1}{m}\log m f_l(x)$) is greater than $\epsilon$. If so, it is believed that the network behaves well on the input $x$. We remark that, the computation of uncertainty example of this kind can be difficult because it lies on the vicinity of the intersection point of all decision boundaries (as illustrated in Fig.~\ref{fig-1}) and such areas are sparse in the input space. %Fig.~\ref{fig-1} (a)) 

Moreover, $s_{U}(x,\hat x)$ requires that the decision of $x$ is significantly far away from $\hat x$. That is, it allows a \emph{user-defined safety risk} $f(\hat x)$ and asks for the network decision to stay away from the risks. 

Based on the expressions, we can also define safe norm balls by requiring that no input in a norm ball satisfies $s(x)<0$. The definition of  maximal safe norm ball can also be extended to this context, and we can define the uncertainty metric the same as that of Definition~\ref{def:localmetric}. Without loss of generality, we will continue use $\mathcal{B}(x,d,p)$ and $Q(s,x, d, p)$ to denote them, respectively. 
As before, a conservative estimation of the maximum radius $\mathcal{B}(x,d,p)$ of safe norm balls can be reduced to the computation of $Q(s,x,d,p)$. Therefore, \textbf{\emph{the study of uncertainty quantification comes for free if we are able to work with the robustness quantification}}.

\subsection{Reachability Quantification} \label{sec:reachability}

For reachability, we can define the following %safety property 
expression: $s_{R}(x)(l) = f_{l}(x) - \epsilon$,
where $\epsilon \in (0,1)$ is a pre-specified threshold for the rigidity of classification probability. Other notions such as $\mathcal{B}(x,d,p)$ and $Q(s,x, d, p)$ follow the discussion in Section~\ref{sec:robustness}.

\section{Risk Quantification Algorithms}

In this section, we consider practical method to calculate the metric $\localmetric(s, x, d, p)$ as in Definition~\ref{def:localmetric}. Instead of basing our method on gradient-based adversarial attack or the formal analysis via encoding of neural networks -- as we discussed in the related work (Section~\ref{sec:related}), we consider derivative-free optimisation methods, which can efficiently search over samples in $\mathcal{B}(x,d,p)$. We remark that, we use robustness -- $\localmetric(s, x, d, p)$ and $\mathcal{B}(x,d,p)$ -- as example, and the algorithms work with uncertainty and reachability.

Given a trained DNN $f$, a property expression $s: \mathbb{R}^m \to \mathbb{R}$, and a genuine $x \in \mathbb{R}^n$, the Lipschitzian metric can be calculated by solving the following optimization problem:
\begin{equation}\label{equ:localoptimize}
	\begin{array}{lllll}
		\min_{\hat{x}} ~&~ 
		%\dfrac{||\hat{x}-x||_p}{|s(\hat{x})-s(x)|}
		w(\hat{x}) ~&~ 
		s.t. ~&~
		||\hat{x}-x||_p \leq d ~&~ 
		\text{ and }~ \hat{x} \in [0,1]^n \\
	\end{array}
\end{equation}
where $w(\hat{x}) = \plainfrac{||\hat{x}-x||_p}{|s(\hat{x})-s(x)|}$. 
The optimization problem contains a non-convex objective 
(due to the non-convexity of DNNs), together with a set of constraints. Note that, for $p\in\{1,2\}$, the constraints include both nonlinear inequality constraints and box-constraints, and for $p=\infty$,
the constraints include only with box-constraints. 

The optimization is based on a composition of the DNN $f$ and the property expression $s$, both of which may be %guaranteed to 
non-differential or not smooth. The analytic form of its first-order derivative is also difficult to get. Methodologically, to achieve the broadest applications, we %intend to introduce 
need a single optimization method that can efficiently estimate different DNN properties for various %types of 
property expressions regardless its differentiability, smoothness, or whether an analytic form of derivative exits. In this regard, instead of using gradient-based method, we take a \textbf{derivative-free optimization framework}.
%that is workable for different types of property expression functions. 
Our optimization solutions are centered around the Mesh Adaptive Direct Search (MADS)~\cite{audet2006mesh}, which is designed for \textbf{black-box optimization} problems for which the functions defining the objective and the constraints are typically seen as black-boxes~\cite{audet2017mesh}. 
It requires no gradient or derivative information but still provides a convergence guarantee to the first-order stationary points based on the Clarke calculus~\cite{Audet2000AnalysisOG,audet2006mesh,audet2017mesh}.

In the following, we will present an 
algorithm for $L_\infty$ norm (Section~\ref{sec:localalgorithm}), enhance the algorithm with tensor-based parallelisation for GPU implementation (Section~\ref{sec:localalgorithmtensor}), and present an algorithm for $L_1$ and $L_2$ norm (Section~\ref{sec:localalgorithmL12}). 

\subsection{\texorpdfstring{$L_\infty$-}-norm Risk Quantification}\label{sec:localalgorithm}

First, we introduce MADS in the context of risk quantification based on $L_\infty$-norm. %based local metric. 
When $p=\infty$, we can transform Eqn.~(\ref{equ:localoptimize}) into the following problem:
\begin{equation}
	\label{equ:localoptimize2}
	\begin{split}
		\min_{\hat{x}} ~~ w(\hat{x})~~~s.t.~~l_d \leq \hat{x} \leq u_d 
	\end{split}
\end{equation}
where $l_d = \max\{x-d, 0\}, u_d = \min\{x+d, 1\}$. 
Instead of presenting the details of MADS \cite{audet2006mesh}, we give its idea. 
Briefly, MADS seeks to improve the current solution by testing points in the neighborhood 
of the current point (the \emph{incumbent}). Each point is one step away in one direction on an iteration-dependent mesh. In addition to these points, MADS can incorporate any search strategy into the optimization to have additional test points. The above process iterates until a stopping condition is satisfied. 

Formally, each iteration of MADS comprises of two stages, a \search{} stage and an optional \poll{} stage. The \search{} stage evaluates a 
number of points proposed by a given search strategy, with the only restriction that the tested points lie on the current mesh. 
The current mesh at the $k$-th iteration is
$M_k = \bigcup_{x \in S_k} \left\{ x + \mesh z \D^{(i)} | z \in \mathbb{N}, \D^{(i)} \in \D \right\}$,
where $S_k \subset \mathbb{R}^n$ is the set of points evaluated since the start of the iteration, $\mesh \in \mathbb{R}_+$ is the \emph{mesh size}, and $\D$ is a fixed matrix in $\mathbb{R}^{\nparams \times n_{\D}}$ whose $n_{\D}$ columns represent viable search directions. We let $\D^{(i)}$ be the $i$-th column of $\D$. In our implementation, we let $\D = \left[\I_\nparams, -\I_\nparams \right]$, where $\I_\nparams$ is the $n$-dimensional identity matrix.

The \poll{} stage is performed if the \search{} fails in finding a point with an improved objective value. \poll{} constructs a \emph{poll set} of candidate points, $P_k$, defined as
$P_k = \left\{x_k + \pollmesh \D^{(i)} | \D^{(i)} \in \D_k \right\},$
where $x_k$ is the incumbent and $\D_k$ is the set of \emph{polling directions} constructed by taking discrete linear combinations of the set of directions $\D$. 
The \emph{poll size} parameter $\pollmesh \ge \mesh$ defines the maximum length of poll displacement vectors $\mesh \D^{(i)}$, for $\D^{(i)} \in \D_k$ (typically, $\pollmesh \approx \mesh \norm{{v}}$).
Points in the poll set can be evaluated in any order, and the \poll{} is opportunistic in that it can be stopped as soon as a better solution is found.
The \poll{} stage ensures theoretical convergence to a local stationary point according to Clarke calculus for nonsmooth functions~\cite{audet2017mesh}. 

If either \search{} or \poll{} succeeds in finding a mesh point with an improved objective value, the incumbent is updated and the mesh size remains the same or is multiplied by a factor $\tau > 1$. 
If neither \search{} or \poll{} is successful, the incumbent does not move and the mesh size is divided by $\tau$. The algorithm proceeds until a stopping criterion is met (e.g., maximum budget of function evaluations). 

\subsection{Tensor-based Parallelisation for \texorpdfstring{$L_\infty$-}-norm Risk Quantification }\label{sec:localalgorithmtensor}

For the problem as in Eqn.~(\ref{equ:localoptimize}), objective function $w(\hat{x})$ includes neural network $f(\hat{x})$. 
Given the availability of tensor-based algorithmic operations in deep learning frameworks such as TesnorFlow, PyTorch, and Caffe, etc, we improve the algorithm described in Section~\ref{sec:localalgorithm} with a tensor-based parallelization, so as to achieve computational efficiency with GPU. As shown in Fig.~\ref{fig-e3}, with a low-end Nvidia GTX1050Ti GPU, to evaluate a 16-layer MNIST DNN on 1,000 images, the time using tensor-based parallelization is 25 times faster than without using one. Specifically, our new algorithm -- enhancing MADS with parallelization -- can improve the speed roughly $(n_k+m_k)/2$ times in terms of DNN inquiry numbers, where $n_k$ and $m_k$ -- to be introduced below -- are such that $n_k$ is around $\geq 2n$ depends on the search strategy and iterations and $m_k$ is at least $\geq n+1$. 

\begin{figure}[h]%[t]
  \centering
  \includegraphics[width=0.56\linewidth]{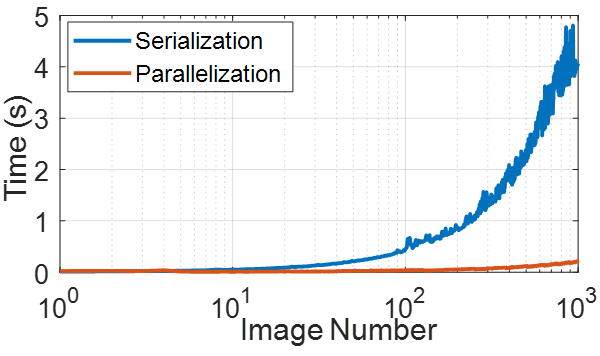}
  \caption{Number of queries to the DNN w.r.t. the number of images, with and without tensor-based parallelization -- a significant motivation for our tensor-based parallelisation algorithm.}
  \label{fig-e3}
\end{figure}

\textbf{Comparing} to the traditional MADS in~\cite{audet2006mesh}, we perform the following improvements in terms of parallelization in both \search{} and \poll{} stages. Algorithm-\ref{alg:tp-mads} provides the pseudo-code for the Parallelised algorithm.

\begin{itemize}
%\vspace{-2mm}
	\item Parallelisation in \search{} Stage: Assuming at $k$-th iteration, there are $n_k$ hyper-points, i.e., $\{x_1^k,x_2^k,..., x_{n_k}^{k}\} \in M_k$, We stack all those hyper-points into a 3-D Tensor $\mathcal{M}^k$ such that $\mathcal{M}^k(i,j,k)$ is the $i$-th element in $x_j^k$. Then we feed $\mathcal{M}^k$ into the GPU to perform the DNN evaluation.
	
	\item Parallelisation in \poll{} Stage: Assuming at $k$-th iteration, there are $m_k$ points in set $P_k = \{x_1^k,x_2^k,..., x_{m_k}^{k}\}$. We stack all those hyper-points into a 3-D Tensor $\mathcal{P}^k$ such that $\mathcal{P}^k(i,j,k)$ is the $i$-th element in $x_j^k$. Then we feed $\mathcal{P}^k$ into the GPU to perform the DNN evaluation.
\end{itemize}

\begin{algorithm}[t!]
	\caption{{Tensor-based Parallelised MADS (TP-MADS)}}
	\label{alg:tp-mads}
	\SetAlgoLined
	\KwIn{Objective function $w(x)$, starting point ${x_0}$, variable constraint $l_d$ and $u_d$\\
		\textbf{Initialization:} $\meshz \leftarrow 2^{-10}$, $\pollmeshz \leftarrow 1$, $k \leftarrow 0$, evaluate $w(x)$ on initial design
	}
	\While{\texttt{fevals} $>$ \texttt{MaxFunEvals} \textbf{ or } $\pollmesh < 10^{-6} $}
	{
		{Stack $\{x_1^k,x_2^k,..., x_{n_k}^{k}\} \in M_k$ into a tensor format $\mathcal{M}^k$;}\\
		{Evaluate $w(x)$ on $\mathcal{M}^k$ via parallelization;}\\
		\If{\textsc{search} is \textsc{not} \emph{successful}}
		{Stack $P_k = \{x_1^k,x_2^k,..., x_{m_k}^{k}\}$ into a tensor format $\mathcal{P}^k$;\\
			Evaluate function $w(x)$ on $\mathcal{P}^k$ via parallelization;} %GPU paralleled computing
		\If{Iteration $k$ is \emph{successful}}
		{Update incumbent $x_{k + 1}$;\\
			\eIf{\poll{} was \emph{successful} }{$\mesh \leftarrow 2 \mesh$, $\pollmesh \leftarrow 2 \pollmesh$;}{$\mesh \leftarrow \frac{1}{2} \mesh$, $\pollmesh \leftarrow \frac{1}{2} \pollmesh$}}
		{Update $k \leftarrow k + 1$}\\}
	\KwOut{ $x_{\text{end}} = \arg \min_k w(x_k)$ and $w(x_{\text{end}})$ }
\end{algorithm}	

\subsection{\texorpdfstring{$L_1$}~ and \texorpdfstring{$L_2$-
}-norm Risk Quantification}\label{sec:localalgorithmL12}

For $L_1$ or $L_2$-norm, 
we need to solve an optimization problem with box-constraint as well as nonlinear inequality constraints, as shown in Eqn.~(\ref{equ:localoptimize}). 
We take an Augmented Lagrangian Algorithm ~\cite{lewis2006generating}
to solve a nonlinear optimization problem with nonlinear constraints, linear constraints, and bounds. Specifically, bounds and linear constraints are handled separately from nonlinear constraints. 
We transform the constrained optimization problem into an unconstrained problem by combining the fitness function and nonlinear constraint function using the Lagrangian and the penalty parameters, as below:
\begin{equation}\label{eqn-49}
	\mathbf{\Theta}(x,\lambda,s) = w(x) - \lambda q \log(q+c(x)),
\end{equation}
where $x\in[0,1]^n$, $\lambda>0$ is a Lagrange multiplier, $q>0$ is a positive shift, and $c(x) = ||x-x_0||_p - d$ where $p \in \{1,2\}$.

Algorithm-\ref{alg:mads-inequality} provides the pseudo-code to solve the $L_1$ and $L_2$-norm risk quantification problem.
The idea of the algorithm is as follows. 
It starts by initialising parameters $\lambda$ and $q$. Then, we minimise a sub-problem, which has fixed values for $\lambda$ and $q$ and is solved by calling Tensor-based Parallelised Mesh Adaptive Direct Search as shown in
Algorithm-\ref{alg:tp-mads}. When the subproblem is minimised to a required accuracy and satisfies feasibility conditions, the Lagrangian estimate (Eqn.~(\ref{eqn-49})) is updated. Otherwise, the penalty parameter $\lambda$ is increased by a penalty factor, together with an update on $q$. This results in a new sub-problem formulation and minimization problem. The above steps (other than the initialisation) are repeated until a stopping criteria are met.

\begin{algorithm}[t!]
	\caption{TP-MADS with Inequality Constraints}
	\label{alg:mads-inequality}
	\SetAlgoLined
	\KwIn{Objective function $w(x)$, starting point ${x_0}$, inequality constraint function $c(x)$, variable constraint $l_d = 0$ and $u_d = 1$}
	\textbf{Initialization:} Initialize $q$ and $\lambda$\\
	\While{Termination criteria not satisfied}
	{
		{Call for Algorithm-1 to solve a Sub-problem Eqn.~(\ref{eqn-49})};\\
		{Update Lagrange multiplier estimate $\lambda$;}\\
		{Update positive shift $q$;}
	}
	\KwOut{ $x_{\text{end}} = \arg \min_{x,\lambda,q} \mathbf{\Theta}(x,\lambda,q)$ and $w(x_{\text{end}})$}
\end{algorithm}

\section{Experimental Results } \label{sec:experiments}

First, in Section~\ref{sec:reachabilityexp}, by comparing  with several state-of-the-art tools on the 
reachability quantification, we show the \textbf{efficiency} of \deepexpress. 
Then, in Section~\ref{sec:robustnessexp}, by conducting robustness quantification on networks of different scales, over datasets MNIST, CIFAR-10 and ImageNet, we show the \textbf{tightness} of results and the \textbf{scalability} of \deepexpress. 
Finally, in Section~\ref{sec:evaluation-KL}, we conduct experiments on \textbf{uncertainty} quantification~\footnote{The software will be found at~\url{https://github.com/TrustAI/DeepQuant}}.

\subsection{Experiments on Reachability Quantification}\label{sec:reachabilityexp}

Three state-of-the-art tools are considered. {\bf Reluplex}~\cite{katz2017reluplex} is an SMT-based method for 
  DNNs with ReLU activations; we apply a bisection scheme to achieve the reachability quantification.
  {\bf SHERLOCK}~\cite{dutta2017output} is a MILP-based method dedicated to reachability quantification on DNNs with ReLU activations.
  {\bf DeepGO}~\cite{RHK2018} is a general reachability quantification tool that can work with a broad range of neural networks including those with non-ReLU activation layers. 
  
We followed the experimental setup in~\cite{dutta2017output} and trained {\bf ten} neural 
networks, including six ReLU networks 
and four Tanh networks (i.e., networks with tanh activations). Note that, neither SHERLOCK nor Reluplex %since 
can work with Tanh networks (i.e., \textsf{tanh-NN-6} to \textsf{tanh-NN-9}).
For ReLU networks, i.e., \textsf{ReLU-NN-0} to \textsf{ReLU-NN-5}, the input 
has two dimensions, i.e., $x \in [0,10]^2$. The input dimensions for \textsf{tanh-NN-6} to \textsf{tanh-NN-9} are gradually increased, from $x\in [0,10]^2$ to $x\in [0,10]^5$. 
For fairness of comparison, we also implement \deepexpress\ in Matlab2018a, running on a Laptop with i7-7700HQ CPU and 16GB RAM. 
The software and hardware setup are made exactly the same as DeepGO~\cite{RHK2018}. Both Reluplex\footnote{\url{https://github.com/guykatzz/ReluplexCav2017}} and SHERLOCK\footnote{\url{https://github.com/souradeep-111/sherlock}} 
are configured to run on a different software platform and 
a more powerful hardware platform --  a Linux workstation with 63GB RAM and a 23-Core CPU. 
We record the running time of each tool when its reachability error is within $10^{-2}$. The comparison results are given in Table~\ref{tab-2}. 

\begin{table*}[ht]
  \centering
  \caption{Comparison with SHERLOCK~\cite{dutta2017output}, Reluplex~\cite{katz2017reluplex} and DeepGO~\cite{RHK2018}.}
  \label{tab-2}
  \scalebox{0.99}{
  \begin{tabular}{l|ccc ccc ccc}
     \toprule
       \vtop{\hbox{\strut NN ID}}
     & \multicolumn{1}{|c}{\vtop{\hbox{Layer$\times$Neuron}}}
         &{\vtop{\hbox{\strut SHERLOCK~~}} }
         &{\vtop{\hbox{\strut Reluplex~~}}} 
         &{\vtop{\hbox{\strut DeepGO~~}}}
         &{\vtop{\hbox{\strut \deepexpress}}}\\ %
         \hline
      \textbf{ReLU-NN-0} &  {1$\times$100} & {1.9s} &  {1m 55s}& {0.4s} & {1.80s} \\

      \textbf{ReLU-NN-1} &  {1$\times$200} & {2.4s} &  {13m 58s }& {1.0s} & {1.56s} \\

      \textbf{ReLU-NN-2} &  {1$\times$500} & {17.8s} &  {Timeout}& {6.8s} & {\bf 1.21s}\\

      \textbf{ReLU-NN-3} &  {1$\times$500} & {7.6s} &  {Timeout}& {5.3s} &  {\bf 1.26s}\\

      \textbf{ReLU-NN-4} &  {1$\times$1000} & {7m 57.8s} &  {Timeout}& {1.8s} &  {\bf 1.21s} \\

      \textbf{ReLU-NN-5} &  {6$\times$250} & {9m 48.4s} &  {Timeout}& {15.1s} & {\bf 2.81s} \\
      
     \textbf{tanh-NN-6 (2-input)} &  {6$\times$250} & {N/A} &  {N/A}& {14.8s} & {\bf 2.93s} \\
    
    \textbf{tanh-NN-7 (3-input)} &  {6$\times$250} & {N/A} &  {N/A}& {58.7s} & {\bf 8.92s} \\
    
    \textbf{tanh-NN-8 (4-input)} &  {6$\times$250} & {N/A} &  {N/A}& {394.1s} & {\bf 20.94s} \\
    
    \textbf{tanh-NN-9 (5-input)} &  {6$\times$250} & {N/A} &  {N/A}& {2680.4s} & {\bf 129.81s} \\

     \bottomrule
  \end{tabular}
   }
\end{table*}

From Table~\ref{tab-2}, our tool \deepexpress\ is consistently better than SHERLOCK and Reluplex.
For the six ReLU-based 
networks, \deepexpress\ has an averaged computation time of around $1.6s$, which has 
{\em 108-fold} and {\em 300-fold} improvement over SHERLOCK and Reluplex (excluding timeouts), respectively. 
Furthermore, {\em the performances of both Reluplex and SHERLOCK are considerably affected by the increase of neuron numbers and layers, while  \deepexpress\ does not}. 
Although both DeepGO and \deepexpress\ can work on Tanh networks, 
DeepGO is significantly more sensitive to the dimension of the input space, with the computation time is nearly exponential w.r.t. the input dimension. Thus, {\em for a neural network with 
high dimensional inputs, \deepexpress\ demonstrates significant superiority over DeepGO}. For example, for the neural network \textsf{tanh-NN-9} (with five input dimensions), \deepexpress\ is nearly {\em 20} times faster.  

In summary, \deepexpress\ exhibits \textbf{better efficiency} than Reluplex, SHERLOCK, and DeepGO. It is less sensitive to the size of network and the input dimensions.

\subsection{Experiments on Robustness Quantification}\label{sec:robustnessexp}

\subsubsection{ACSC-Xu Networks}
The first experiment is performed on a 5-input and 5-output ACSC-Xu neural networks~\cite{katz2017reluplex}. We aim to validate the accuracy -- or tightness -- of \deepexpress\ on robustness quantification.
From this section, all experiments are conducted on a PC with i7-7700HQ CPU, 16GB RAM, and GPU GTX1050Ti. DNNs are trained with the Neural Network Toolbox in MATLAB2018a. 
The ACSC-Xu neural network is trained on a simulated dataset and includes 5 fully-connected layers, ReLU activation functions, and overall it contains 300 hidden neurons~\cite{katz2017reluplex}. 
The five input variables of ACAS-Xu neural network are shown in Fig.~\ref{fig-acscxu} ( which are obtained from various kinds of sensors~\cite{bunel2017piecewise}), where $\rho$ (m) presents Distance from ownship to intruder, $\theta$ (rad) is Angle to intruder relative to ownship heading direction, $\psi$ (rad) shows Heading angle of intruder relative to ownship, $v_{own}$ (m/s) and $v_{int}$ (m/s) display Speed of ownship and intruder respectively.

\begin{figure}[ht]
	\centering
	\includegraphics[width=0.56\linewidth]{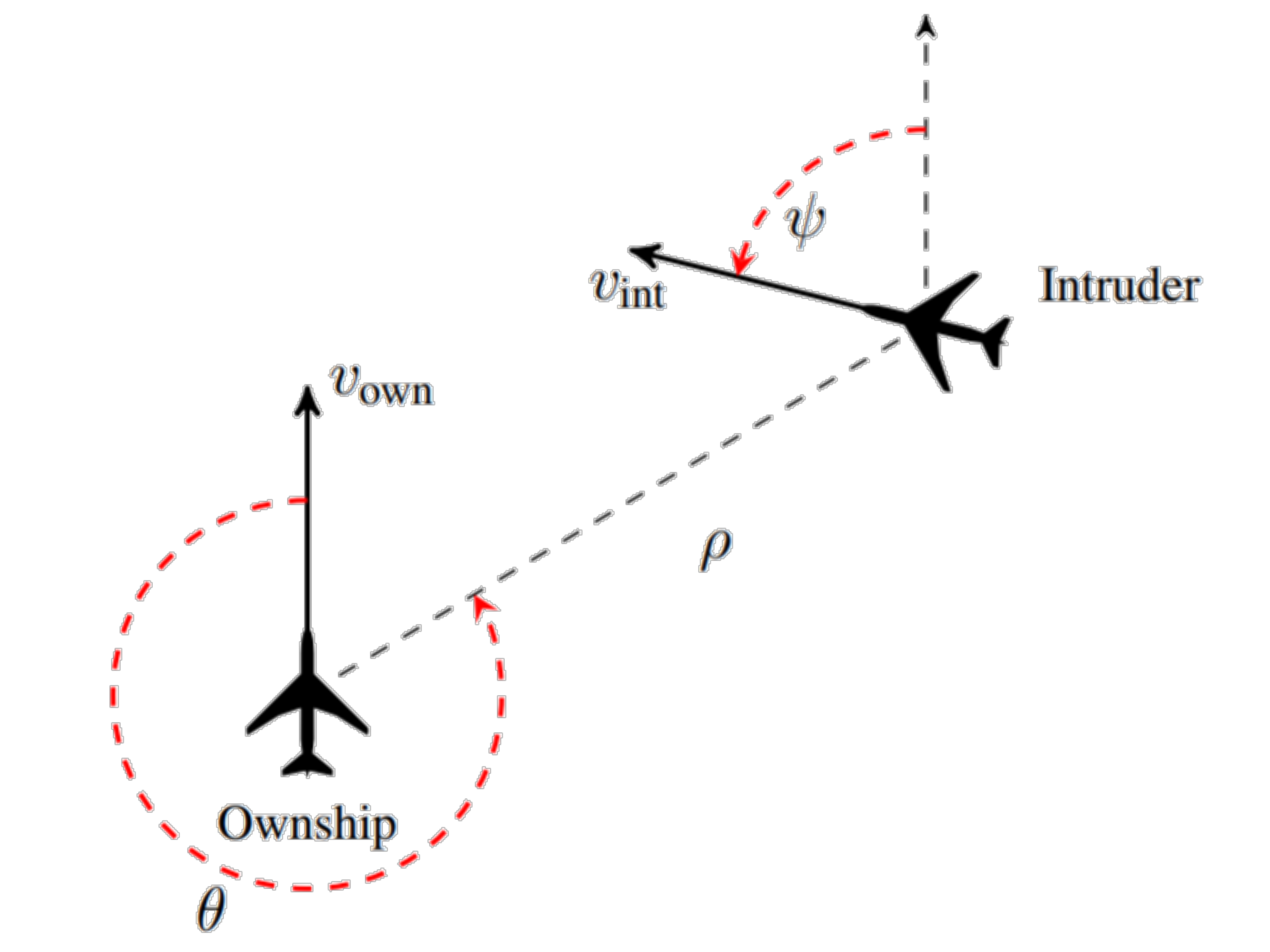}
	\caption{Geometry for ACAS Xu Horizontal Logic Table (from~\cite{katz2017reluplex}).}
	\label{fig-acscxu}
\end{figure}

We adapt the safety verification tool \deepgo~\cite{RHK2018} for the computation of  ground-truth robustness quantification values. Moreover, we implement the other baseline method -- a random sampling (\textbf{RS}) method, which uniformly samples $5\times 10^5$ images in a given norm ball. 

Fig.~\ref{fig-3} (a) and Fig.~\ref{fig-3} (b) present the comparison on the accuracy and the query number, respectively,  
over different norm distance ($L_\infty$, $L_1$ and $L_2$). We see that \deepexpress\ can almost reach the ground-truth accuracy value computed by DeepGO (as in Fig.~\ref{fig-3} (a)), but with much less number of queries (as in Fig.~\ref{fig-3} (b)). Precisely, 
\deepexpress\ takes around $2\times 10^3$ DNN queries, while \deepgo\ requires around $1.3\times10^4$ DNN queries -- {\bf 6} times difference. Moreover, \deepexpress\ performs much better than {\bf RS}, on both the tightness and the efficiency. 
In other word, this experiment exhibits \textbf{both the tightness of the result and the efficiency of the computation}.

\begin{figure}[ht]
\centering
	    \centering
	    \begin{minipage}{0.47\linewidth}
	        \centering
		    \includegraphics[width=1\linewidth]{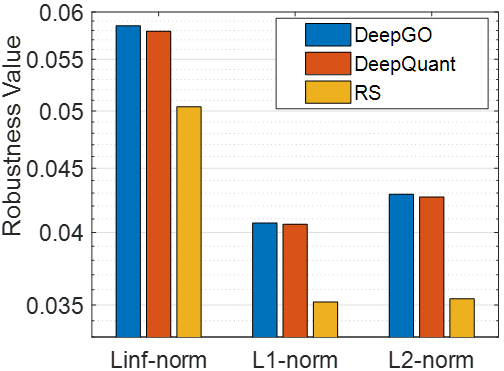}
		    \text{(a)}
	    \end{minipage}%
	    \hspace{3mm}
	    \begin{minipage}{0.47\linewidth}
		    \centering
		    \includegraphics[width=1\linewidth]{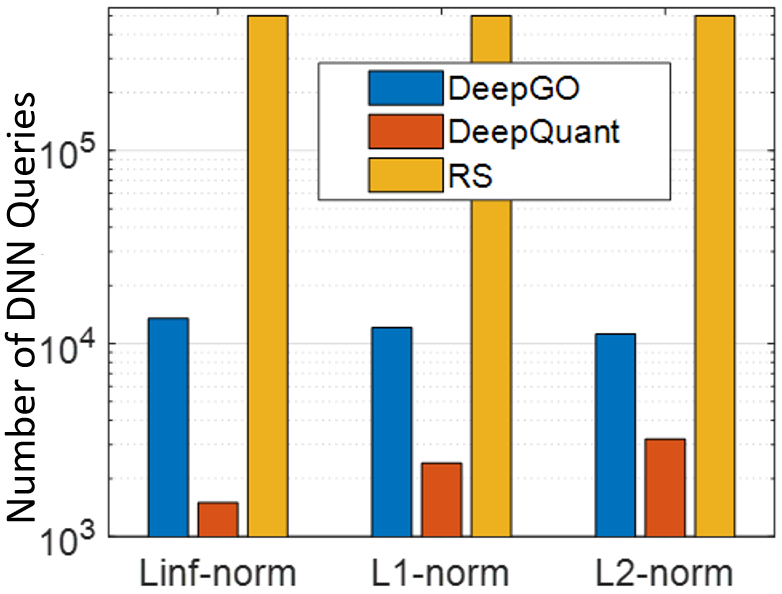}
		    \text{(b)}
	    \end{minipage}
	\centering
	\caption{{\bf(a)} Accuracy comparison of Robustness Quantification for $L_\infty$, $L_1$ and $L_2$-norm on the ACSC-Xu network. 
	{\bf~~(b)} Comparing DNN inquiry numbers when using different methods for robustness quantification on the ACSC-Xu network.
	}
    \label{fig-3}
\end{figure}

\subsubsection{MNIST and CIFAR-10 Networks}

We train a 9-layer DNN on MNIST dataset and a 10-layer DNN on CIFAR-10 dataset. They achieve 99.4\% and 78.3\% testing accuracy respectively, which are comparable to the state-of-the-art~\cite{rodrigob} without data augmentation or other layer modifications. 
Table~\ref{tab-MNIST2} and Table~\ref{tab-CIFAR2} present the model structures of MNIST DNN and CIFAR-10 DNN respectively. Table~\ref{tab:datasets} shows the detail information about training dataset and training parameter setups on MNIST and CIFAR-10.

\begin{table*}[ht]
  \centering
  \caption{Structure of MNIST DNN. The pipeline consists of Convolution layer (Conv), Batch-Normalization layer (Batch), and Fully Connected layer (FC). }
  \label{tab-MNIST2}
  \scalebox{0.99}{
  \begin{tabular}{l|ccc ccc ccc}
     \toprule
       \vtop{\hbox{\strut Layer Type}}
         & \multicolumn{1}{|c}
         {\vtop{\hbox{ Number of Channels}} }
         &{\vtop{\hbox{\strut Filter Size}} }
         &{\vtop{\hbox{\strut Stride Value}} }
         &{\vtop{\hbox{\strut Activation}}} 
         &{\vtop{\hbox{\strut Output Size}}}\\
         \hline
        {Conv1} &  {$1$} &  {$3\times 3\times 16$}& {1} & {ReLU} & {$28\times 28 \times 16$} \\
       \hline
        {Conv2 + Batch} &  {$16$} &  {$3\times 3\times 32$}& {1} & {ReLU} & {$28\times 28 \times 32$} \\
       \hline
       {Conv3 + Batch} &  {$32$} &  {$3\times 3\times 64$} & {1} & {ReLU} & {$28\times 28 \times 64$} \\
       \hline
       {Conv4 + Batch} &  {$64$} &  {$3\times 3\times 128$}& {1 }  & {ReLU} & {$28\times 28 \times 128$} \\
       \hline
       {Dropout} &  {-} &  {-}& {-}  & {-} & {$28\times 28 \times 128$} \\
      \hline
       {FC} &  {-} &  {-}& {-} & {ReLU} & {256} \\
      \hline
       {Dropout} &  {-} &  {-}& {-} & {-} & {256} \\
      \hline
      {FC} &  {-} &  {-}& {-} & {Softmax} & {10} \\
     \bottomrule
  \end{tabular}
   }
\end{table*}

\begin{table*}[ht]
  \centering
  \caption{Structure of CIFAR-10 DNN. The pipeline consists of Convolution layer (Conv), Max Pooling (MaxPool), and Fully Connected layer (FC). }
  \label{tab-CIFAR2}
  \scalebox{0.99}{
  \begin{tabular}{l|ccc ccc ccc}
     \toprule
       \vtop{\hbox{\strut Layer Type}}
         & \multicolumn{1}{|c}
         {\vtop{\hbox{ Number of Channels}} } 
         &{\vtop{\hbox{\strut Filter Size}} }
         &{\vtop{\hbox{\strut Stride Value}} }
         &{\vtop{\hbox{\strut Activation}}} 
         &{\vtop{\hbox{\strut Output Size}}}\\
         \hline
       {Conv1} &  {$3$} &  {$3\times 3\times 32$}& {1} & {ReLU} & {$32\times 32 \times 32$} \\
       \hline
       {Conv2} &  {$32$} &  {$3\times 3\times 32$}& {1} &{ReLU}& {$30\times 30 \times 32$} \\
      %{Batch1} &  {$32$} &  {-}& {-} & {-} \\
       \hline
       {MaxPool} & {$32$} &  {$2\times 2\times 32$}& {2} &{-}& {$29\times 29 \times 32$}\\
       \hline
       {Conv3} &  {$32$} &  {$3\times 3\times 64$} & {1} & {ReLU} & {$29\times 29 \times 64$} \\
      %{Batch2} &  {$64$} &  {-}& {-} & {-} \\
       \hline
       {Conv4} &  {$64$} &  {$3\times 3\times 64$}& {1}  & {ReLU} & {$27\times 27 \times 64$} \\
      %{Batch3} &  {$128$} &  {-}& {-} & {-} \\
       \hline
       {MaxPool} & {$64$} &  {$2\times 2\times 64$}& {2} &{-}& {$26\times 26 \times 64$}\\
       \hline
        {Dropout} &  {-} &  {-}& { -} & {-} & {$26\times 26 \times 64$} \\
       \hline
       {FC} &  {-} &  {-}& {-} & {ReLU} & {512} \\
        \hline
       {FC} &  {-} &  {-}& {- } & {Softmax} & {10} \\
     \bottomrule
  \end{tabular}
   }
\end{table*}

\begin{table*}[ht]
    \caption{Detailed information about MNIST and CIFAR-10 dataset.}
    \centering
    \begin{tabular}{|l|c|c|c|c|}
    \hline
         Dataset & Training Set Size & Testing Set Size & Testing Accuracy & Parameter Optimization Setup  \\
         \hline
        MNIST & $60,000$ & $10,000$  & $99.41\%$ & Max Epochs=35, Batch=128, optimizer=SGDM \\  \hline
        CIFAR-10 &  $50,000$ & $10,000$ & $78.30\%$ & Epochs=50, Batch=128, optimizer=SGD \\
        \hline
    \end{tabular}
    \label{tab:datasets}
\end{table*}

Fig.~\ref{fig-3more} shows the robustness quantification results for $L_1$, $L_\infty$-norm and $L_2$-norm respectively on 10 input images (selected from testing dataset) for the MNIST network.
The norm balls for these three different robustness quantification are set as $d = 250$, $d = 0.3$ and $d = 8$ respectively. For random sampling 
we sampled 1,000,000 images in the norm ball to evaluate $\localmetric(s, x, d, p)$ based on Definition~\ref{def:localmetric}.  
 We can see that, \deepexpress\ performs consistently better  while using tens of times less DNN queries. 
 Please note, in this experiment, DeepGO is not included due to its limitation on scalability.
From the Fig.~\ref{fig-3more}, we can see that the proposed robustness quantification method is consistently better than random sampling. Moreover, in our experiment, even through random sampling approach samples $10^6$ images, it still cannot achieves an accurate robustness evaluation.

\begin{figure*}[ht]
\centering
	    \centering
	    \begin{minipage}{0.32\linewidth}
	        \centering
		    \includegraphics[width=1\linewidth]{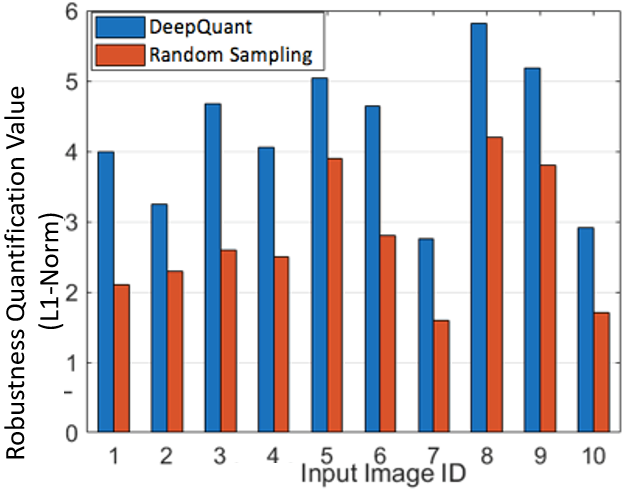}
		    \text{(a)}
	    \end{minipage}%
	    \hspace{2mm}
	    \begin{minipage}{0.32\linewidth}
		    \centering
		    \includegraphics[width=1\linewidth]{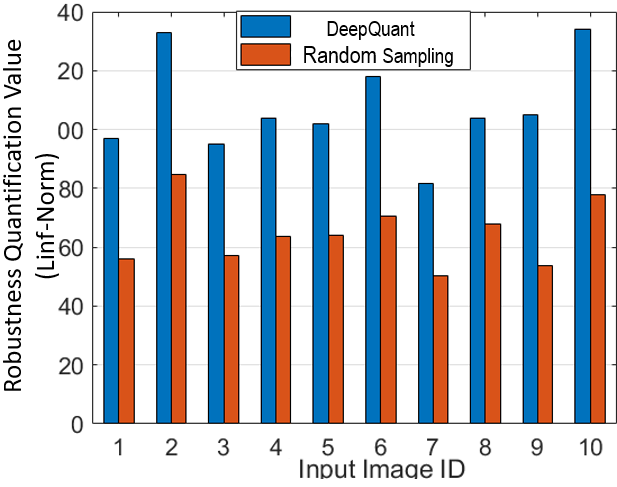}
		    \text{(b)}
	    \end{minipage}
	    \hspace{2mm}
	   \begin{minipage}{0.32\linewidth}
		    \centering
		    \includegraphics[width=1\linewidth]{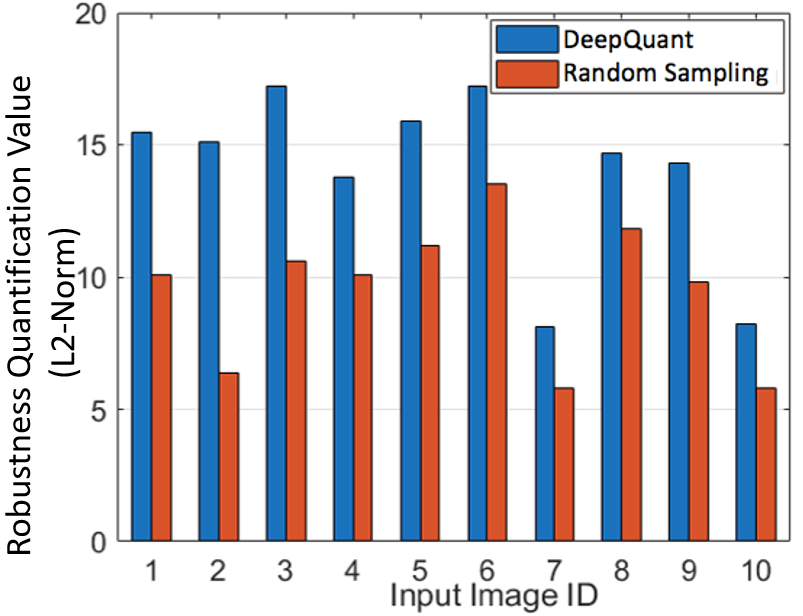}
		    \text{(c)}
	    \end{minipage}
	\caption{{\bf(a)} Comparison of $L_1$-norm robustness quantification on a MNIST deep neural network, $d = 250$. 
	{\bf~~(b)} Comparison of $L_\infty$-norm robustness quantification on a MNIST deep neural network, $d = 0.3$.
 	{\bf~~(c)} Comparison of $L_2$-norm robustness quantification on a MNIST deep neural network, $d = 8$. 
	}
    \label{fig-3more}
\end{figure*}

We might also be interested in 
\textbf{targeted robustness quantification}, which essentially measures the hardness of fooling input images 
into a given target label. For the CIFAR-10 network, Fig.~\ref{fig-4} (a) gives the evaluation results for label-1 as the target label. We can see that label-3 is the most robust while label-7 is the least robust. 

\begin{figure}[ht]
\centering
	   \begin{minipage}{0.47\linewidth}
		    \centering
		    \includegraphics[width=1\linewidth]{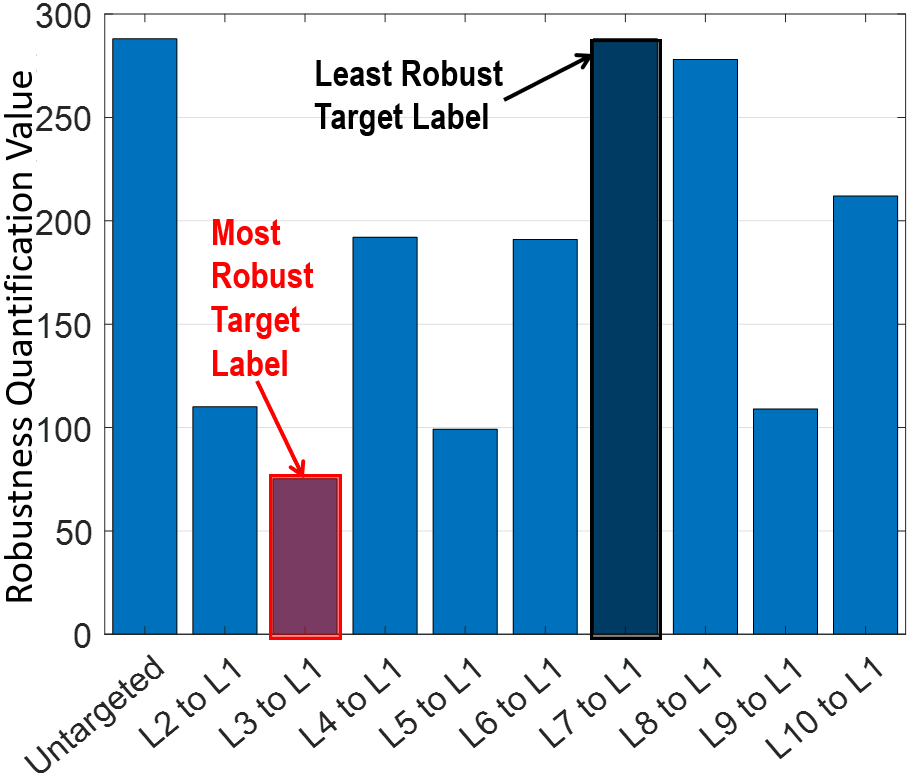}
		    		    \text{(a)}
	    \end{minipage}
	    \hspace{3mm}
	    \begin{minipage}{0.47\linewidth}
	        \centering
	        \includegraphics[width=1\linewidth]{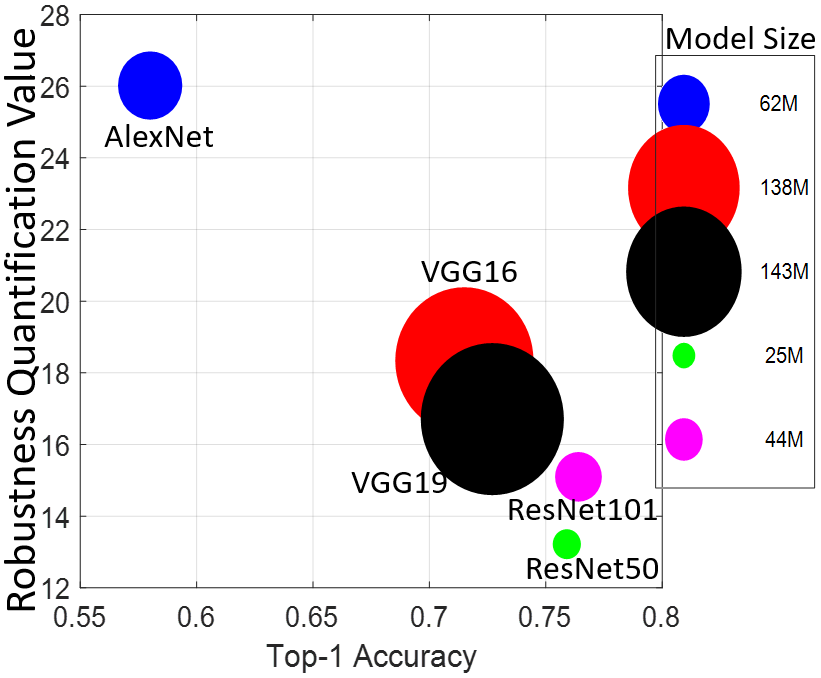}
	        \text{(c)}
	    \end{minipage}
	\caption{
	{\bf ~(a)} Targeted robustness evaluation using \deepexpress\ on a CIFAR-10 DNN, with label-1 as the target label.   We use LX-L1 on the X-axis to indicate the targeted robustness value from label-X to label-1.
    {\bf (b)} Comparison of robustness values of five ImageNet DNNs on an input image using \deepexpress.
	}
    \label{fig-4}
\end{figure}

Moreover, Fig.~\ref{fig-4more} gives some images returned by \deepexpress\ (i.e., $\hat x$ in Eqn.~(\ref{equ:localrobustness}) while evaluating the  robustness of MNIST and CIFAR-10 networks. The MNIST images in Fig.~\ref{fig-4more} (a) 
are generated when performing $L_\infty$, $L_1$ and $L_2$-norm robustness evaluation. 
The CIFAR-10 images in Fig.~\ref{fig-4more} (b) 
are images found by \deepexpress\ when gradually increasing the norm ball radius (i.e., $d$ in $\localmetric(s, x, d, p)$) from $0.1$ to $0.4$. It shows that the visual difference w.r.t. input image 
becomes more obvious for a larger $d$ due to the monotonicity of local robustness value w.r.t. the norm-ball radius. Those images essentially exhibit where the confidence interval decreases the fastest in their corresponding norm balls. We remark that, \textbf{they are different from adversarial examples,  and showcase potentially
important robustness risks of a network}.

\begin{figure}[ht]
\centering
	    \centering
	    \begin{minipage}{0.47\linewidth}
	        \centering
		    \includegraphics[width=1\linewidth]{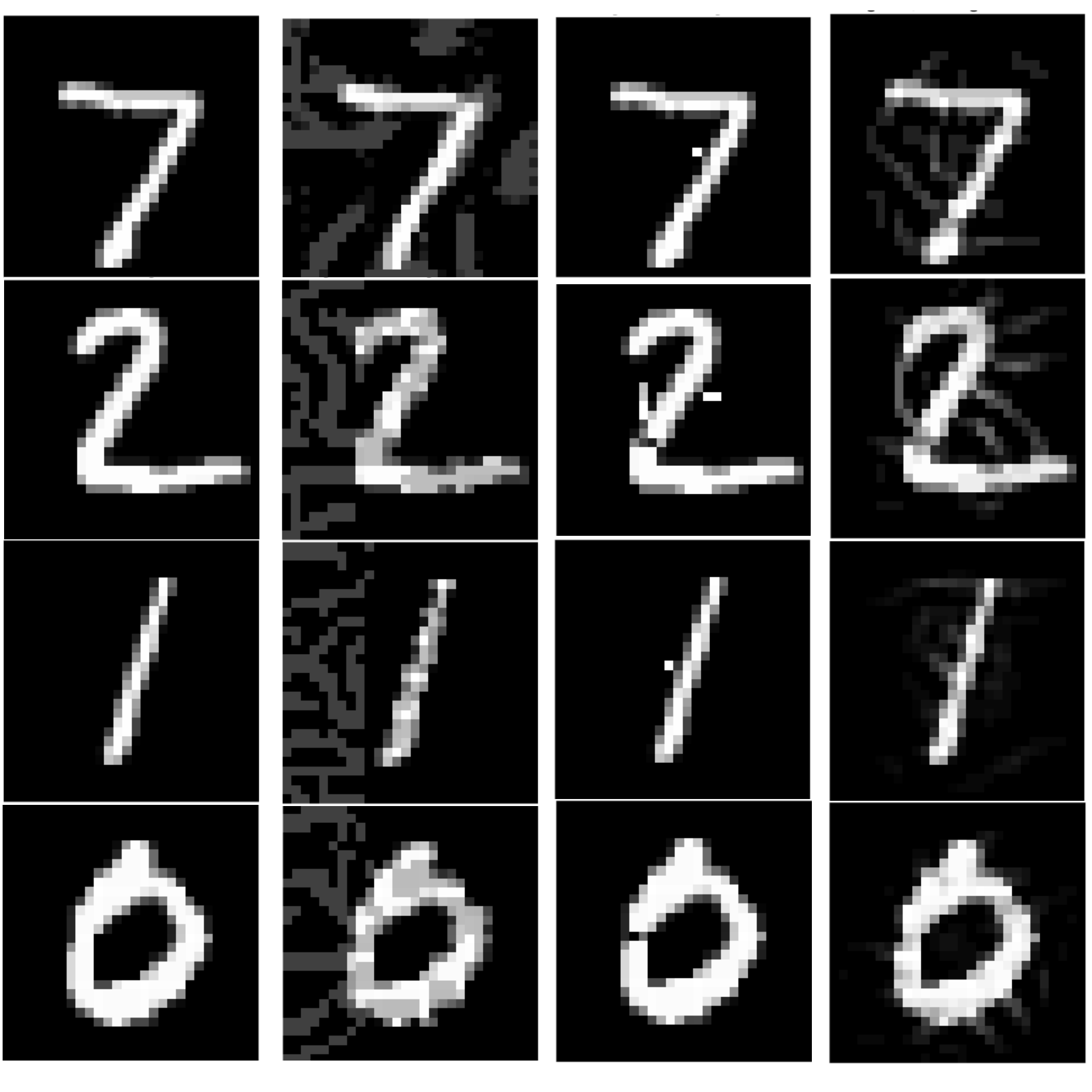}
		    \text{(a)}
	    \end{minipage}%
	    \hspace{3mm}
	    \begin{minipage}{0.47\linewidth}
		    \centering
		    \includegraphics[width=1\linewidth]{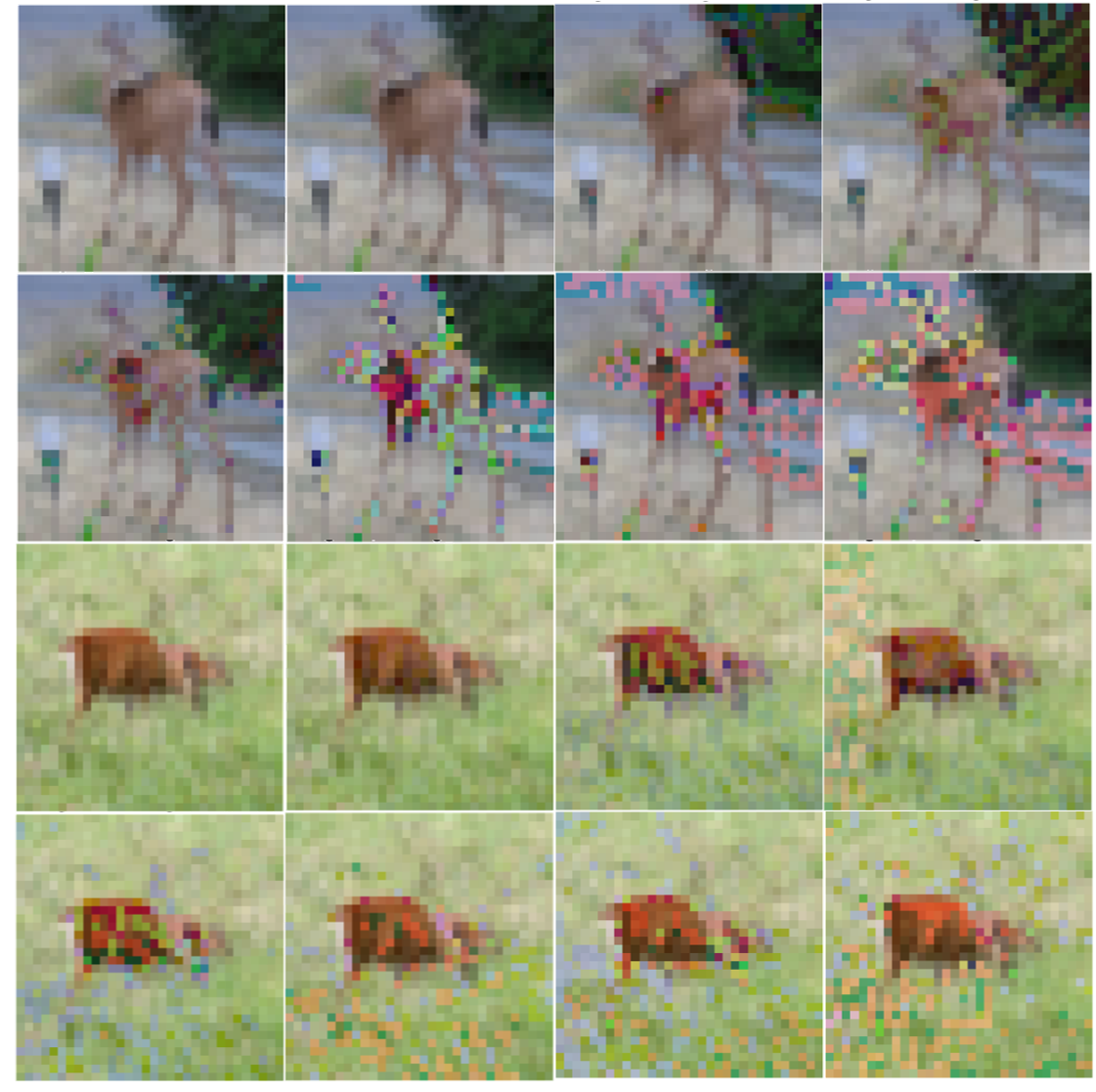}
		    \text{(b)}
	    \end{minipage}
	\centering
	\caption{
	{\bf(a)} MNIST images in the first two rows are examples returned by \deepexpress, i.e., $\hat x$ in Eqn.~(\ref{equ:localrobustness}). {\em From Left to Right: it indicates the original image, images by using $L_\infty$-norm, $L_1$-norm, and $L_2$-norm robustness metric.}
	{\bf~~(b)} CIFAR-10 images in the last two rows are examples returned by \deepexpress\ for an input image by increasing the $L_\infty$-norm ball radius $d = 0.1:0.05:0.4$.
	}
    \label{fig-4more}
\end{figure}

That is, \deepexpress\ can be used to study  variants of safety properties. 

\subsubsection{ImageNet Networks:}

In Fig.~\ref{fig-4} (b), we measure the robustness of five ImageNet models, including AlexNet (8 layers), VGG-16 (16 layers), VGG-19 (19 layers), ResNet50 (50 layers), and ResNet101 (101 layers), on a $L_\infty$-norm ball for a chosen feature (i.e., a $50\times 50$ square). We can see that, for this local norm space and the chosen feature, ResNet-50 achieves best robustness and AlexNet is the least robust one. This experiment shows the \textbf{scalability} of \deepexpress\ in working with large-scale networks. 

In addition, we also presents a case study showing how to use \deepexpress\ to guide the {\bf Design of Robust DNN Models} by using robustness quantification. 
We train six DNNs on MNIST dataset (see their model structure details in Fig.~\ref{fig-nnstructure}, which mainly includes convolution layer (conv), batch-normalization layer (batchnorm), and fully connected layer (fc)).  
The DNNs range from with shallow layers (e.g., DNN-1) to deep layers (e.g., DNN-6). We randomly choose 100 images and use \deepexpress\ to evaluate their $L_\infty$-norm robustness. 
Table-\ref{tab-linfMNIST} presents the result of five input images and the mean robustness values. Based on the robustness statistics, a DNN builder can choose suitable DNNs for different tasks with balance of accuracy and robustness. For example, for a non-critical application that requires high accuracy, DNN-6 is the most suitable one; for a safety-critical application, DNN-2 is a good choice; DNN-4 and DNN-5 however have good balance on accuracy and robustness.
\begin{figure}[ht]
	\centering
	\includegraphics[width=0.9\linewidth]{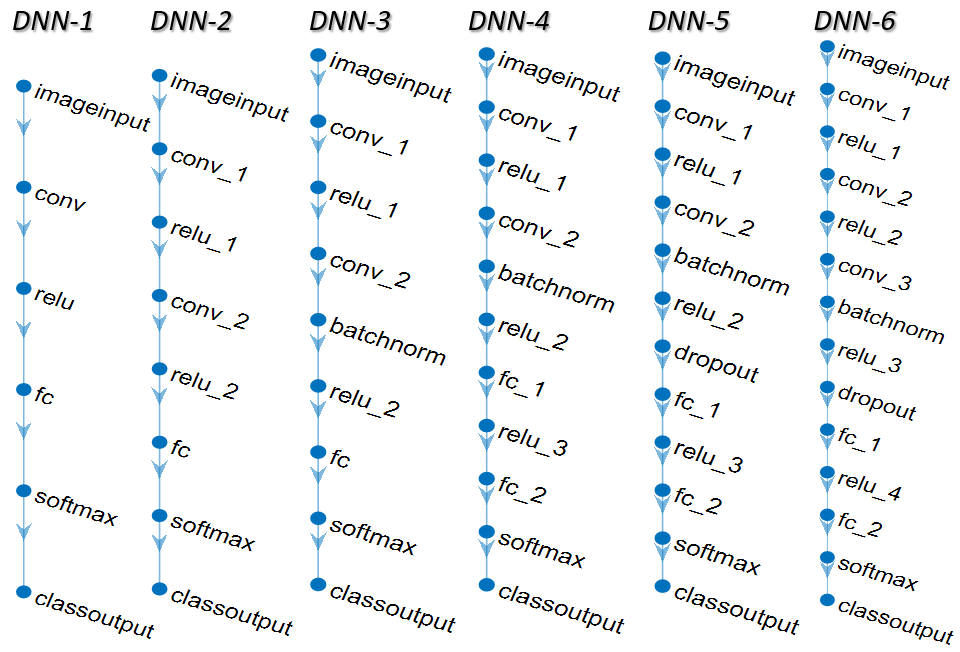}
	\caption{Model Structures of MNIST DNNs from DNN-1 to DNN-6. The filter size of conv\_1, conv\_2 and conv\_3 are $3 \times 3\times 32$, $3 \times 3\times 64$, and $3 \times 3\times 128$ respectively. The probability of dropout is 0.5.}
	\label{fig-nnstructure}
\end{figure}

\begin{table}[ht]
  \centering
  \caption{$L_\infty$-norm robustness quantification results on six MNIST DNNs given five input images.}
  \label{tab-linfMNIST}
  \scalebox{0.9}{
  \begin{tabular}{l|ccc ccc ccc}
     \toprule
     & \multicolumn{1}{|c}{\vtop{\hbox{\strut Img-1}}}
         & \vtop{\hbox{\strut Img-2}}
         &{\vtop{\hbox{\strut Img-3}}} &{\vtop{\hbox{\strut Img-4}}} &{\vtop{\hbox{\strut Img-5}}}&\multicolumn{1}{|c} {\vtop{\hbox{\strut Mean}}}&\multicolumn{1}{|c} {\vtop{\hbox{\strut Test Acc.}}}\\ \hline
         
         \textbf{DNN-1} &  {45.90 }& {93.73 } & {30.44 } &  {39.76 }& {93.33 } & \multicolumn{1}{|c} {60.63 }& \multicolumn{1}{|c} {97.75\%} \\

      \textbf{DNN-2} &  {\bf 35.24 }& {\bf 21.66 } & {\bf 19.79 } &  {\bf 26.82 }& {\bf 57.30 } & \multicolumn{1}{|c} {\bf 32.16 }& \multicolumn{1}{|c} {97.95\%} \\

      \textbf{DNN-3} &  {87.13 }& {69.40 } & {78.31 } &  {84.71 }& {100.30 } & \multicolumn{1}{|c} {83.97 }& \multicolumn{1}{|c} {98.38\%}\\

      \textbf{DNN-4} &  {42.07 }& {46.42 } & {69.90 } &  {46.86 }& {63.42 } & \multicolumn{1}{|c} {53.73 }& \multicolumn{1}{|c} {99.06\%} \\

      \textbf{DNN-5} &  {53.68 }& {54.17 } & {82.78 } &  {41.21 }& {65.75 } & \multicolumn{1}{|c} {59.52 }& \multicolumn{1}{|c} {99.16\%} \\

      \textbf{DNN-6} &  {96.91 }& {70.13} & {75.53 } &  {64.63 }& {80.19 } & \multicolumn{1}{|c} {77.48 }& \multicolumn{1}{|c} {\bf 99.41\%} \\

     \bottomrule
  \end{tabular}
   }
\end{table}

\subsection{Experiments on Uncertainty Quantification} \label{sec:evaluation-KL}

We adopt the same MNIST and CIFAR-10 networks as those in Section~\ref{sec:robustnessexp}. The detailed experimental setup can be found in Table~\ref{tab:datasets}. 

In Fig.~\ref{fig-5} (a), we first showcase what is an uncertainty example. The top row is for a true image which has a high KL divergence to uniform distribution, and the bottom is for the uncertainty image found by \deepexpress\ in a $L_\infty$-norm ball ($d = 0.4$).
From human perception, the uncertainty image should certainly have the same 
label as the original one.

\begin{figure*}[ht]
\centering
\begin{minipage}{0.28\linewidth}
		    \centering
  \includegraphics[width=1\linewidth]{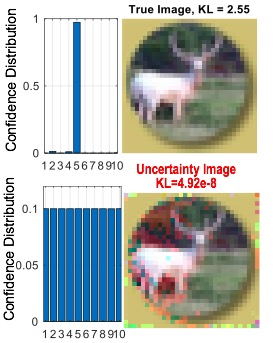} 
		    \text{(a)}
	    \end{minipage}
	   \begin{minipage}{0.36\linewidth}
		    \centering
		    \includegraphics[width=1\linewidth]{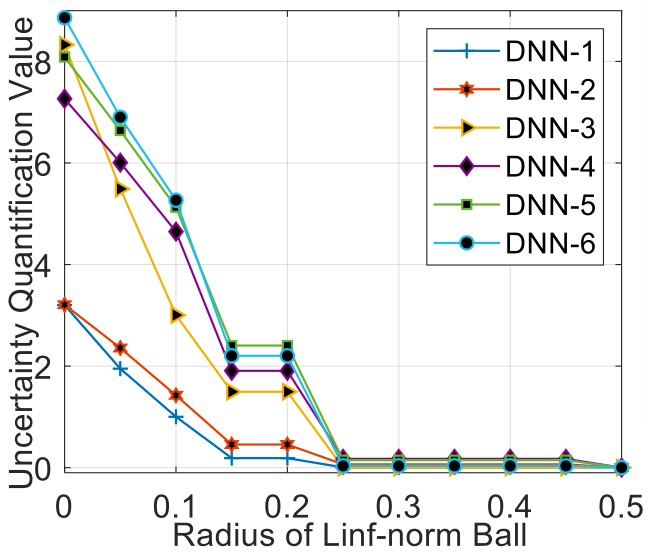}
		    \text{(b)}
	    \end{minipage}
	    \begin{minipage}{0.34\linewidth}
		    \centering
		    \includegraphics[width=1\linewidth]{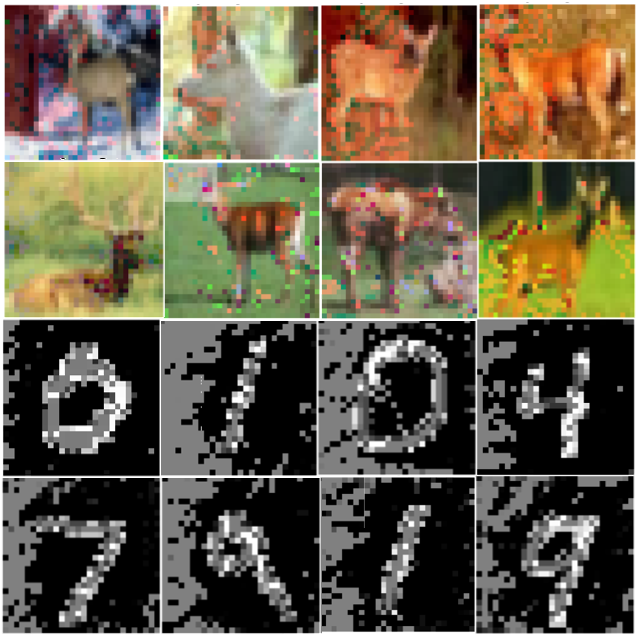}
		    \text{(c)}
	    \end{minipage}
    \caption{ 
    {\bf~~(a)} KL divergence between DNN's output distribution with uniform distribution for a true image and an uncertainty image.
       {\bf~~(b)} Uncertainty values quantified by \deepexpress\ on six MNIST DNNs on a given $L_\infty$-norm ball.
     {\bf~~(c)} Uncertainty examples on MNIST and CIFAR-10 dataset in a $L_\infty$-norm ball of radius $d = 0.4$.
    }
    \label{fig-5}
\end{figure*}

In Fig.~\ref{fig-5} (b), we use \deepexpress\ to quantify uncertainty for the six MNIST networks (see Fig.~\ref{fig-nnstructure} 
for the details of their structures),  
while gradually increasing norm-ball radius from 0 to 0.5. We see that, the uncertainty of networks vastly worsens with the increase of norm-ball radius. At $d=0.15$, DNN-1 and DNN-2 show worse uncertainty than other networks. 
Fig.~\ref{fig-5} (c) gives some uncertainty examples captured by \deepexpress\ on MNIST and CIFAR-10 neural networks. 

In Fig.~\ref{fig-8}, we visualise several intermediate images obtained during a search for an uncertainty image in a $L_\infty$-norm ball with $d = 0.1$. From  left to right, the true image is perturbed by \deepexpress\ with an optimization objective of minimising the KL divergence. With the perturbations, the generated images have gradually increased uncertainty 
related to this specific input.  When the KL divergence is reduced to
0, the network is completely confused and does not know how to classify the uncertainty example. Thus, \deepexpress\ is the very first tool that can insightfully and automatically reveal this new, yet very 
important, safety property in the decision process of a network.

\begin{figure*}[ht]
\centering
    \begin{minipage}{1\linewidth}
		    \centering
    \includegraphics[width=1\linewidth]{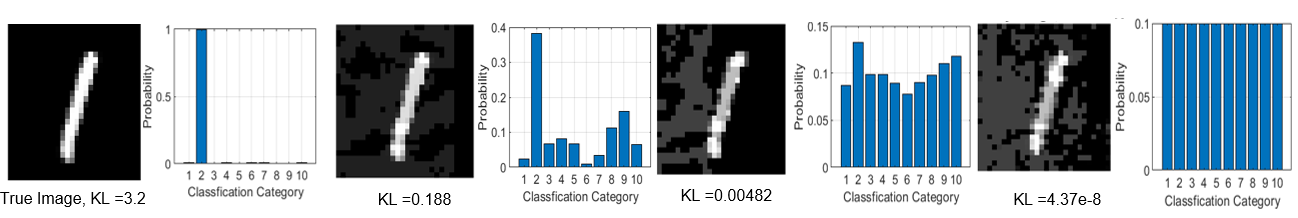}
	\end{minipage}
    \begin{minipage}{1\linewidth}
		    \centering
	\includegraphics[width=1\linewidth]{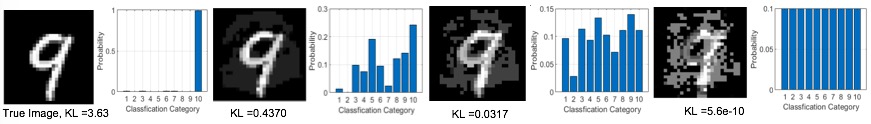} 
	\end{minipage}
    \caption{Intermediate images obtained during the searching for an uncertainty image in a $L_\infty$-norm ball with $d = 0.1$. From left to right, the KL divergence is gradually decreased.}
    \label{fig-8}
\end{figure*}

\section{Conclusion}

This paper presents a novel method~\deepexpress -- based on a generic Lipschitz metric and a derivative-free optimisation algorithm -- to quantify a set of safety risks,  including a new risk called uncertainty example. Comparing with state-of-the-art methods, our method not only can work on a broad range of risks but also returns tight result comparing to the ground truth. Our tool \deepexpress\ is optimized by tensor-based parallelisation, which could run efficiently on GPUs, and thus is scalable to work with large-scale networks including MNIST, CIFAR-10 and ImageNet models. We envision that this paper provides an initial yet important attempt towards the risk quantification concerning the safety of DNNs.

%%%%%% References:
\bibliographystyle{unsrt}  
\bibliography{Arxiv-DeepQuant.bbl}

\end{document}